\documentclass[twoside]{article}

\usepackage[accepted]{aistats2026}
\usepackage[round]{natbib}

\usepackage{amsmath,amssymb}     
\usepackage{amsthm}

\theoremstyle{definition}
\newtheorem{theorem}{Theorem}[section]

\newtheorem{definition}{Definition}[section]

\newtheorem{assumption}{Assumption}[section]
\newtheorem{remark}{Remark}[section]

\usepackage[ruled,vlined]{algorithm2e}
\SetAlgoNlRelativeSize{-1}
\SetAlgoInsideSkip{0pt}
\SetKwInput{Input}{Input}
\SetKwInput{Output}{Output}

\usepackage{graphicx}
\usepackage{subcaption}
\usepackage{mdframed}
\usepackage{makecell}

\usepackage{booktabs}
\usepackage{multirow}

\usepackage{subcaption}

\usepackage{enumitem}

\begin{document}

\runningtitle{Dashed Line Defense: Plug-And-Play Defense Against Adaptive Score-Based Query Attacks}

\runningauthor{Yanzhang Fu, Zizheng Guo, Jizhou Luo}

\twocolumn[

\aistatstitle{Dashed Line Defense: Plug-And-Play Defense\\Against Adaptive Score-Based Query Attacks}

\aistatsauthor{ Yanzhang Fu\textsuperscript{*} \\ \texttt{research@fyzdalao.xyz} \And Zizheng Guo\textsuperscript{*} \\ \texttt{zguo.research@gmail.com} \And Jizhou Luo\textsuperscript{*\textdagger} \\ \texttt{luojizhou@hit.edu.cn} }

\aistatsaddress{
  \textsuperscript{*} Harbin Institute of Technology \qquad 
  \textsuperscript{\textdagger} \textit{Corresponding Author}
}

]

\begin{abstract}
  % The Abstract paragraph should be indented 0.25 inch (1.5 picas) on
  % both left and right-hand margins. Use 10~point type, with a vertical
  % spacing of 11~points. The \textbf{Abstract} heading must be centered,
  % bold, and in point size 12. Two line spaces precede the
  % Abstract. The Abstract must be limited to one paragraph.
  Score-based query attacks pose a serious threat to deep learning models by crafting adversarial examples (AEs) using only black-box access to model output scores, iteratively optimizing inputs based on observed loss values. While recent runtime defenses attempt to disrupt this process via output perturbation, most either require access to model parameters or fail when attackers adapt their tactics. In this paper, we first reveal that even the state-of-the-art plug-and-play defense can be bypassed by adaptive attacks, exposing a critical limitation of existing runtime defenses. We then propose Dashed Line Defense (DLD), a plug-and-play post-processing method specifically designed to withstand adaptive query strategies. By introducing ambiguity in how the observed loss reflects the true adversarial strength of candidate examples, DLD prevents attackers from reliably analyzing and adapting their queries, effectively disrupting the AE generation process. We provide theoretical guarantees of DLD’s defense capability and validate its effectiveness through experiments on ImageNet, demonstrating that DLD consistently outperforms prior defenses—even under worst-case adaptive attacks—while preserving the model’s predicted labels.
\end{abstract}

\section{Introduction}

Deep neural networks (DNNs) have achieved remarkable success across various domains, yet they remain vulnerable to adversarial examples (AEs) — inputs that differ only slightly from legitimate samples but cause the model to make incorrect predictions~\citep{goodfellow2015explaining}. This vulnerability raises serious security concerns, especially given the rapid advancement of black-box score-based query attacks (SQAs). These attacks leverage zeroth-order optimization (ZOO) to craft AEs by querying only the model’s output scores, without requiring access to internal parameters, making them highly practical and hard to defend against.

To mitigate these threats, researchers have explored two main categories of defenses against SQAs: training-time and runtime defenses. Training-time defenses, known as adversarial training (AT)~\citep{madry2018towards}, enhance robustness by incorporating AEs during model training but are costly and often impractical when data or model access is limited. Runtime defenses instead attempt to disrupt the attacker’s optimization by perturbing outputs during prediction~\citep{qin2021random, RND}, avoiding retraining but often degrading accuracy or requiring model modifications.
Among the recent runtime defenses, the state-of-the-art Adversarial Attack on Attackers (AAA)~\citep{AAA} stands out for its plug-and-play design, offering strong protection claims without altering model predictions or requiring access to internal parameters.

However, as we reveal in Section~\ref{sec:AAAAA}, AAA remains vulnerable to adaptive attacks, exposing a fundamental limitation of current runtime defenses against black-box adversaries: they often underestimate attacker adaptivity by assuming static or uninformed adversaries, while the true effectiveness of any defense depends critically on its ability to withstand adaptive adversaries who design attack strategies to target the defense itself.  
This observation echoes findings from the white-box setting~\citep{tramer2020adaptive}, where defenses that ignore attacker adaptivity were shown to provide only illusory robustness.

\begin{figure}
  \centering
  \includegraphics[width=0.31 \textwidth]{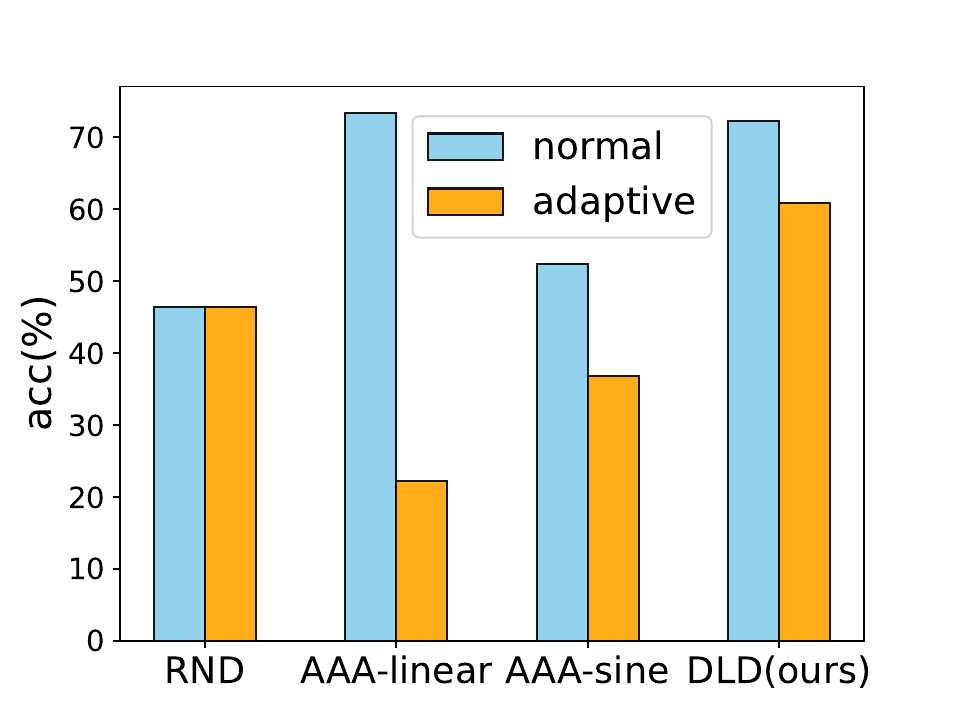}
  \caption{Under-attack accuracy of four defenses, with blue and orange bars indicating normal and adaptive attacks, respectively. Note that RND reduces accuracy on non-adversarial samples.}
  \label{fig:intro}
\end{figure}

Building on the observed weakness of the state-of-the-art defense, we propose Dashed Line Defense (DLD), a plug-and-play post-processing method that preserves predicted labels and remains effective against adaptive attacks. DLD employs a non-continuous mapping between the attacker-observed loss and the true loss (illustrated later in Figure~\ref{fig:dld}), introducing ambiguity in how the observed loss reflects the adversarial effectiveness of AE candidates. This ambiguity disrupts smooth optimization paths and greatly increases difficulty for attackers to analyze and bypass the defense.

We provide a theoretical analysis characterizing DLD’s defense strength against both standard and adaptive attacks. Experiments on the ImageNet dataset show that AAA, despite the state-of-the-art runtime defense, suffers severe degradation under adaptive attacks (see Figure~\ref{fig:intro} and Section~\ref{sec:results}), whereas DLD maintains high accuracy even under worst-case adaptive scenarios. Furthermore, DLD consistently outperforms the prior plug-and-play approach RND~\citep{RND}, achieving higher accuracy while preserving model predictions. These results highlight the necessity of designing defenses that account for attacker adaptivity, which is DLD’s central motivation.

Our main contributions are threefold:
\begin{itemize}
    \item We reveal a fundamental limitation of the state-of-the-art runtime defense (AAA) against adaptive SQA tactics, exposing a concrete gap between claimed and actual protection under adaptivity.
    \item We propose Dashed Line Defense (DLD), a novel plug-and-play post-processing method that confuses the attack optimization process and is difficult to bypass.
    \item We provide a formal analysis that characterizes DLD’s defense strength, and through experiments on ImageNet, show that DLD outperforms prior plug-and-play defenses (including AAA and RND) even under worst-case adaptive attacks.
\end{itemize}

\section{Related Work}

\textbf{Generating Adversarial Examples} began with pioneering works~\citep{szegedy2014intriguing, goodfellow2015explaining}, which showed DNNs are vulnerable to small perturbations. In white-box settings, where model gradients are accessible, attacks like C\&W~\citep{carlini2017towards} generate AEs using gradient-based optimization. In black-box settings, early methods relied on substitute models to approximate gradients~\citep{papernot2017practical, li2020towards, xie2019improving}.

In 2017, \cite{chen2017zoo} proposed a ZOO-based attack that outperformed substitute model-based methods in black-box settings. This breakthrough sparked extensive subsequent research employing either gradient estimation-based ZOO~\citep{ilyas2018black, bhagoji2018practical, ilyas2019prior, al2020sign} or random search-based ZOO~\citep{guo2019simple, alzantot2019genattack, Square}.
However, these methods mainly focus on efficiently generating the next candidate during optimization, while overlooking challenging scenarios like non-convex landscapes or stochastic noise, leading to suboptimal performance against defenses.

\textbf{Defenses to Adversarial Examples} can be categorized based on how they interact with the model: 
(1) modifying the training process, 
(2) altering the model’s internal structure, or 
(3) adding an external plug-and-play module.

The first approach, AT~\citep{madry2018towards, tramer2018ensemble}, augments training data with AEs and retrains the model. 
Although effective, it reduces clean accuracy, incurs high computational cost~\citep{shafahi2019adversarial}, and often generalizes poorly to unseen attacks~\citep{tramer2020adaptive}. 
Its effectiveness remains largely empirical \citep{schmidt2018adversarially}, and recent theoretical work on data debugging \citep{guo2024datadebugging} suggests that certifying the effects of training-time data manipulations like AT may be fundamentally difficult.

The second approach enhances robustness by modifying the model’s architecture or computation.
Examples include weight noise injection~\citep{qin2021random}, feature denoising~\citep{xie2019feature}, and randomized smoothing~\citep{cohen2019certified}.
These methods can improve robustness against standard attacks but often rely on gradient obfuscation and fail under adaptive attacks~\citep{athalye2018obfuscated}.

The third approach adds plug-and-play modules without modifying the model. 
Pre-processing defenses perturb the input before prediction~\citep{byun2022effectiveness, RND}, but the resulting change to output scores is indirect and unpredictable, potentially altering the model’s prediction. 
Post-processing defenses, such as AAA~\citep{AAA}, modify the output scores after prediction, ensuring the predicted label remains unchanged. 
However, as shown in this paper, even such state-of-the-art post-processing defenses can be bypassed by adaptive attacks.

While there also exist methods that detect attacks via historical queries \citep{chen2020stateful, li2024query, park2025mind},
their focus is somewhat orthogonal to this work. Such detection inherently limits the number of queries an attacker can make, forcing attackers to prioritize query efficiency over time efficiency.

\textbf{Adaptive Attacks on Defenses} have evolved in response to defense proposals.
Many defenses relying on obfuscated gradients were soon broken by adaptive attacks~\citep{athalye2018obfuscated}.
\cite{tramer2020adaptive} outlined general principles for designing adaptive attacks, and \cite{yao2021automated} developed a tool to automate the process of crafting such attacks.
These successful adaptive attacks primarily target white-box settings with full model access. While black-box attacks can be incorporated into adaptive strategies, query efficiency is often neglected. To our knowledge, adaptive attacks in black-box settings remain largely underexplored.

\section{Preliminaries}

\subsection{Score-Based Black-Box Attack}

We consider a black-box classifier $f: \mathcal{X} \rightarrow \mathcal{Y}$, where $f_y(\mathbf{x})$ denotes the output score for class $y$ given input $\mathbf{x}$, and the predicted label is $f(\mathbf{x}) = \mathop{\arg \max}_{y} f_y(\mathbf{x})$. 
Only the output scores $f_y(\mathbf{x})$ for each class $y$ can be accessed through queries, while the model’s internal parameters and gradients are inaccessible. 
Let $\mathcal{N}(\mathbf{x}_0,\epsilon) = \{\mathbf{x} \mid \Vert \mathbf{x} - \mathbf{x}_0 \Vert \le \epsilon \}$ denote the neighborhood around a clean input $\mathbf{x}_0$. An adversarial example is formally defined as follows:

\begin{definition}
Let $f: \mathcal{X} \rightarrow \mathcal{Y}$ be a classifier, $\mathbf{x}_0 \in \mathcal{X}$ a clean input, and $\epsilon_\text{n}$ a perturbation budget. An input $\mathbf{x} \in \mathcal{X}$ is a $(f, \mathbf{x}_0, \epsilon_\text{n})$-adversarial example (AE) if
\[
\mathbf{x} \in \mathcal{N}(\mathbf{x}_0,\epsilon_\text{n}) \quad \text{and} \quad f(\mathbf{x}) \ne f(\mathbf{x}_0).
\]
\end{definition}

Given a correctly classified input $\mathbf{x}_0 \in \mathcal{X}$ with ground-truth label $y \in \mathcal{Y}$, a query budget $n$, and a perturbation budget $\epsilon_\text{n}$, the goal of a black-box score-based query attack (SQA) is to craft a $(f, \mathbf{x}_0, \epsilon_\text{n})$-AE using no more than $n$ queries to $f$.

\interfootnotelinepenalty=10000

Since the model is black-box and only the output scores are accessible, the attack is typically modeled as a ZOO process that minimizes the margin loss:
\begin{equation}
\label{eq:margin_loss_1}
\mathcal{L}(\mathbf{x},y) = f_y(\mathbf{x}) - \max_{t\neq y}f_t(\mathbf{x}),
\end{equation}
where $y \in \mathcal{Y}$ is a reference label that defines the margin. It is typically chosen as the predicted label of the initial input $\mathbf{x}_0$. A negative margin indicates that the model no longer favors label $y$, and thus serves as a criterion for attack success
\footnote{
This loss is designed for untargeted attacks. In targeted attacks, the label $t$ is a fixed target class chosen by the attacker rather than being arbitrary. This paper focuses on untargeted settings.
}.
If $y$ is unspecified, we use the shorthand notation:
\begin{equation}
\mathcal{L}(\mathbf{x}) = f_{f(\mathbf{x})}(\mathbf{x}) - \max_{t\neq f(\mathbf{x})}f_t(\mathbf{x}).
\end{equation}

Each iteration, the attacker uses a generator $G_{(\mathbf{x}_0,\epsilon_\text{n})}(\cdot): \mathcal{X} \rightarrow \mathcal{N}(\mathbf{x}_0,\epsilon_\text{n})$ to produce a new candidate $\mathbf{x}^{\text{new}}$.
It may implement heuristic strategies to explore the neighborhood of $\mathbf{x}_0$ within the allowed perturbation budget.
Then, the attacker compares $\mathbf{x}^{\text{new}}$ with current best sample $\mathbf{x}^k$ and performs the update:
\begin{equation}
\mathbf{x}^{k+1} = \left\{
\begin{aligned}
\mathbf{x}^{\text{new}}, \quad & \mathcal{L}(\mathbf{x}^{\text{new}},y) < \mathcal{L}(\mathbf{x}^k,y) \\
\mathbf{x}^k, \quad & \text{otherwise}
\end{aligned}
\right.
\label{min-decide}
\end{equation}

This iterative procedure is summarized in Algorithm~\ref{alg:standard}. The attack is deemed successful if it finds some $\mathbf{x}$ such that $\mathcal{L}(\mathbf{x}, y) < 0$; otherwise, it fails when the query budget $n$ is exhausted without success.

Note that generating $\mathbf{x}^{\text{new}}$ using $G_{(\mathbf{x}_0,\epsilon_\text{n})}$ may involve multiple queries to the model, so the number of iterations can be fewer than the query budget $n$.
We write $G = G_{(\mathbf{x}_0,\epsilon_\text{n})}$ for brevity whenever the context permits.

\begin{algorithm}[tb]
\caption{standard-SQA$(\mathbf{x}_0,G_{(\mathbf{x}_0,\epsilon_\text{n})})$}
\label{alg:standard}
\Input{clean sample $\mathbf{x}_0$, example generator $G_{(\mathbf{x}_0,\epsilon_\text{n})}$}
\Output{AE candidate $\mathbf{x}$}

$y_0 \gets f(\mathbf{x}_0)$, $\mathbf{x} \gets \mathbf{x}_0$\;

\While{$\mathcal{L}(\mathbf{x},y_0)>0$ \textbf{and} not out of query budget}{
    $\mathbf{x}_{\text{new}} \gets G_{(\mathbf{x}_0,\epsilon_\text{n})}(\mathbf{x})$\;
    
    \If{$\mathcal{L}(\mathbf{x}_{\text{new}},y_0) < \mathcal{L}(\mathbf{x},y_0)$}{
        $\mathbf{x} \gets \mathbf{x}_{\text{new}}$\;
    }
}

\KwRet{$\mathbf{x}$}
\end{algorithm}

\subsection{Plug-and-Play Defense}

Under black-box constraints, defenders can adopt plug-and-play defenses that perturb model outputs to mislead attackers, causing observed loss values to differ from the true ones. These defenses fall into two categories: preprocessing and postprocessing.

Given input $\mathbf{x}$, let $\mathcal{L}_{\text{real}}(\mathbf{x})$ be the true margin loss and $\mathcal{L}_{\text{observe}}(\mathbf{x})$ the attacker-observed loss. Preprocessing applies a (typically randomized) transformation $\mathcal{D}_{\text{pre}}$ before inference:
\[
\mathcal{L}_{\text{observe}}(\mathbf{x}) = \mathcal{L}_{\text{real}}(\mathcal{D}_{\text{pre}}(\mathbf{x})).
\]
Postprocessing perturbs the loss directly after prediction:
\[
\mathcal{L}_{\text{observe}}(\mathbf{x}) = \mathcal{D}_{\text{post}}(\mathcal{L}_{\text{real}}(\mathbf{x})).
\]

Preprocessing may harm prediction accuracy on benign inputs, since $f(\mathcal{D}_{\text{pre}}(\mathbf{x}))$ may differ from $f(\mathbf{x})$. In contrast, postprocessing preserves the original prediction while modifying only the loss.

\subsection{Robustness Measure}

We adopt a formal robustness notion proposed by \cite{katz2017towards} to characterize the stability of a classifier’s output scores under small input perturbations.

\begin{definition}[Global Robustness]
\label{def:global_robust}
A classifier $f: \mathcal{X} \rightarrow \mathcal{Y}$ is said to be $(\delta, \epsilon)$-globally-robust in an input region $D$ \textbf{iff}
\begin{align*}
\forall \mathbf{x}_1, \mathbf{x}_2 \in D, \ 
&\Vert \mathbf{x}_1 - \mathbf{x}_2 \Vert \le \delta \\
&\Rightarrow \ \forall y \in \mathcal{Y}, \ 
\vert f_y(\mathbf{x}_1) - f_y(\mathbf{x}_2) \vert < \epsilon.
\end{align*}
\end{definition}

This definition requires that for any two inputs $\mathbf{x}_1, \mathbf{x}_2 \in D \subseteq \mathcal{X}$ that are close to each other, the model produces similar confidence scores for all classes. 
We will later use this notion to formalize assumptions about model behavior and to quantify the cost and guarantees of the proposed defense.

\section{Defense Against Adaptive Attack}

\subsection{Adaptive Attack on Existing Work}
\label{sec:AAAAA}

We begin by examining why and how the state-of-the-art defense method AAA \citep{AAA} is vulnerable to adaptive attacks.

AAA is the first plug-and-play post-processing defense, whose loss mapping function $\mathcal{D}_{\text{post}}$ is defined as
\begin{equation*}
\mathcal{D}_{\text{post}}(\mathcal{L}) = 2\lfloor \mathcal{L}/\tau \rfloor\tau + \tau - \mathcal{L},
\end{equation*}
where $\tau$ is a constant determining the length of each monotonic interval.
%The image of $\mathcal{D}_{\text{post}}$ is shown in figure \ref{fig:AAA}.
The blue curve in Figure~\ref{fig:AAA} shows the image of this $\mathcal{D}_{\text{post}}$.

\begin{figure}
  \centering
  \includegraphics[width=0.3 \textwidth]{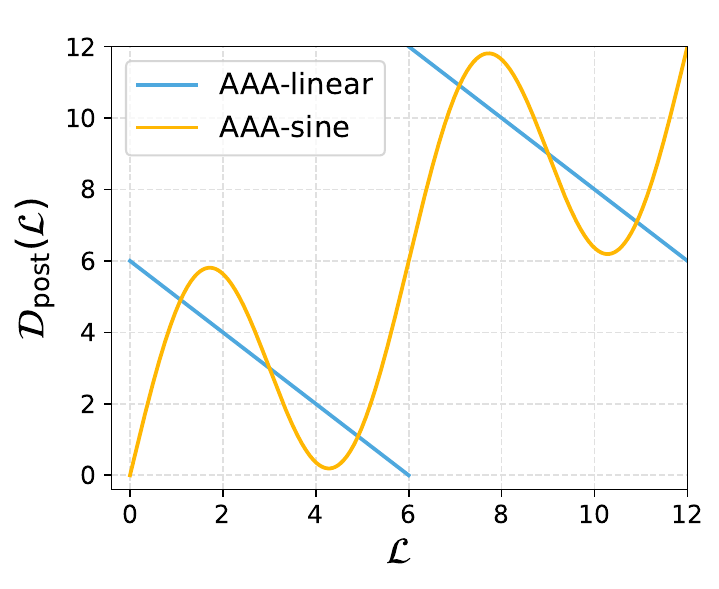}
  \caption{$\mathcal{D}_{\text{post}}$ of AAA}
  \label{fig:AAA}
\end{figure}

This version, known as AAA-linear, was later found to be vulnerable after reviewer feedback: by reversing the descent direction within each interval, it enables attackers to adapt by switching to maximization. Specifically, Eq.~\ref{min-decide} can be replaced with:
\begin{equation}
\mathbf{x}^{k+1} = \left\{
\begin{aligned}
\mathbf{x}^{k}, \ \ \mathcal{L}(\mathbf{x}^{\text{new}},y) < \mathcal{L}(\mathbf{x}^k,y) \\
\mathbf{x}^{\text{new}}, \ \ \mathcal{L}(\mathbf{x}^{\text{new}},y) \ge \mathcal{L}(\mathbf{x}^k,y) \\
\end{aligned}
\right.
\label{max-decide}
\end{equation}

To address this vulnerability, AAA-sine was proposed. With a constant $\alpha$ (0.7 by default), $\mathcal{D}_{\text{post}}$ is defined as
\begin{equation*}
\mathcal{D}_{\text{post}}(\mathcal{L}) = \mathcal{L} - \alpha\tau\sin(\pi(1-(2\mathcal{L}- (2\lfloor\mathcal{L}/\tau\rfloor+1)\tau   )/\tau)).
\end{equation*}
By combining increasing and decreasing intervals, this $\mathcal{D}_{\text{post}}$ aims to defend against both minimization and maximization steps, as shown in Figure~\ref{fig:AAA}. AAA-sine is therefore claimed to be robust to adaptive attacks.
However, this design still has fundamental flaws. Each monotonic interval in $\mathcal{D}_{\text{post}}$ is only capable of misleading either minimization or maximization steps, but not both simultaneously.
To exploit this weakness, we design a simple adaptive tactic (Algorithm~\ref{alg:reverse}) that alternates between minimization and maximization. The attack begins with standard minimization until no improvement is observed for $t$ consecutive iterations. It then switches to maximization using Eq.~\ref{max-decide} until the next stagnation, after which the direction-reversing cycle repeats.
Despite its simplicity, this direction reversing tactic significantly degrades AAA's effectiveness. For instance, the under-attack accuracy of AAA-sine drops from 61.7\% to 41.5\% (see Section~\ref{sec:results}).

\begin{algorithm}[tb]
\caption{reverse-tactic-SQA$(\mathbf{x}_0,G,t)$}
\label{alg:reverse}
\Input{clean sample $\mathbf{x}_0$, sample generator $G$, threshold $t$}
\Output{AE candidate $\mathbf{x}$}

Initialize $reverse \gets \texttt{False}$, $cnt \gets 0$\;
$y_0 \gets f(\mathbf{x}_0)$, $\mathbf{x} \gets \mathbf{x}_0$, $loss \gets \mathcal{L}(\mathbf{x},y_0)$\;

\While{$loss>0$ \textbf{and} not out of query budget}{
    $\mathbf{x}_{\text{new}} \gets G(\mathbf{x})$\;
    $loss_{\text{new}} \gets \mathcal{L}(\mathbf{x}_{\text{new}},y_0)$\;
    
    \If{$\neg ((loss_{\text{new}}<loss) \oplus reverse)$}{
        $cnt \gets cnt + 1$\;
    }
    
    \If{$cnt > t$}{
        $cnt \gets 0$, $reverse \gets \neg reverse$\;
    }
    
    \If{$(loss_{\text{new}}<loss) \oplus reverse$}{
        $cnt \gets 0$, $\mathbf{x} \gets \mathbf{x}_{\text{new}}$, $loss \gets loss_{\text{new}}$\;
    }
}

\KwRet{$\mathbf{x}$}
\end{algorithm}

\subsection{Dashed Line Defense}
\label{sec:DLD}

To address the weaknesses of post-processing defenses discussed above, we propose a novel method called Dashed Line Defense (DLD), where $\mathcal{D}_{\text{post}}$ is specifically designed to interfere with both minimization and maximization steps. 
A key insight behind DLD is that $\mathcal{D}_{\text{post}}$ must be non-smooth; otherwise, any monotonic interval of a continuous and smooth $\mathcal{D}_{\text{post}}$ would allow attackers to infer a reliable gradient direction, rendering the direction-reversal tactic (Algorithm~\ref{alg:reverse}) effective. The DLD mechanism is formally defined below.

\begin{definition}
\label{def:DLD}
Given $\tau > 0$, $h \in [0,1]$, and $S \subseteq [0,1)$, a $(\tau, h, S)$-DLD is a post-processing module whose loss mapping $\mathcal{D}_{\text{post}}$ is defined as follows:

\begin{equation}
\label{eq:bias}
\mathcal{L}_{\text{bias}}(\mathcal{L}; \tau) = \left\lfloor \mathcal{L} / \tau
\right\rfloor \cdot \tau
\end{equation}
\begin{equation}
\label{eq:high}
\begin{split}
\mathcal{L}_{\text{high}}(\mathcal{L}; \tau, h) 
&= \mathcal{L}_{\text{bias}}(\mathcal{L}; \tau) + h\tau \\
&\quad + (1 - h)(\mathcal{L} - \mathcal{L}_{\text{bias}}(\mathcal{L}; \tau))
\end{split}
\end{equation}
\begin{equation}
\label{eq:low}
\mathcal{L}_{\text{low}}(\mathcal{L}; \tau, h) = \mathcal{L}_{\text{bias}}(\mathcal{L}; \tau) + h\tau - h\left(\mathcal{L} - \mathcal{L}_{\text{bias}}(\mathcal{L}; \tau)\right)
\end{equation}
\begin{equation}
\mathcal{D}_{\text{post}}(\mathcal{L}; \tau, h, S) = 
\begin{cases}
\mathcal{L}_{\text{high}}(\mathcal{L}; \tau, h) & \text{\!if } \frac{\mathcal{L} - \mathcal{L}_{\text{bias}}(\mathcal{L}; \tau)}{\tau} \in S \\
\mathcal{L}_{\text{low}}(\mathcal{L}; \tau, h) & \text{otherwise}
\end{cases}
\end{equation}

\end{definition}

For readability, fixed parameters such as $\tau$, $h$, and $S$ are omitted from notations like $\mathcal{D}_{\text{post}}(\mathcal{L}; \tau, h, S)$ when clear from context.
The loss mapping function $\mathcal{D}_{\text{post}}$ in DLD is defined over periodic intervals of length $\tau$, with each interval starting at a point given by Eq.~(\ref{eq:bias}). In practice, $S$ is typically chosen as a union of narrow subintervals, e.g., $(\epsilon, 2\epsilon) \cup (3\epsilon, 4\epsilon) \cup \dots \cup (1 - \epsilon, 1)$, where $\epsilon$ is a small positive constant. 
As illustrated in Fig.~\ref{fig:dld}, the image of $\mathcal{D}_{\text{post}}$ resembles a staircase-like pattern formed by dashed “less-than” symbols. The upper half of each symbol corresponds to $\mathcal{L}_{\text{high}}$, and the lower half corresponds to $\mathcal{L}_{\text{low}}$.

The implementation of $\mathcal{D}_{\text{post}}$ is straightforward.
After the model produces the output scores, a lightweight post-processing module modifies them in a label-preserving manner: it adjusts the difference between the top two scores—that is, the margin loss—to match the output of $\mathcal{D}_{\text{post}}$.
This incurs negligible computational overhead.
Since the predicted label is not changed, DLD does not modify the decision boundary and therefore lies outside the scope of decision-based (label-only) attacks (e.g., \cite{cheng2019query,chen2020hopskipjumpattack}); in such settings, DLD neither helps nor harms defense.

Note that $\vert \mathcal{L} - \mathcal{D}_{\text{post}}(\mathcal{L}) \vert \le \tau$, which bounds the impact on the model’s prediction confidence.
In our implementation, we modify only the score of the highest-scoring class while keeping all other scores unchanged, which suffices to realize the desired margin transformation.
As a result, under this implementation, if a model $f$ is $(\delta,\epsilon)$-globally-robust, then its DLD-defended version remains $(\delta,\epsilon + 2\tau)$-globally-robust. Increasing $\tau$ may enhance defense but could also introduce greater distortion to the model’s outputs.

\begin{figure}
  \centering
  \includegraphics[width=0.3\textwidth]{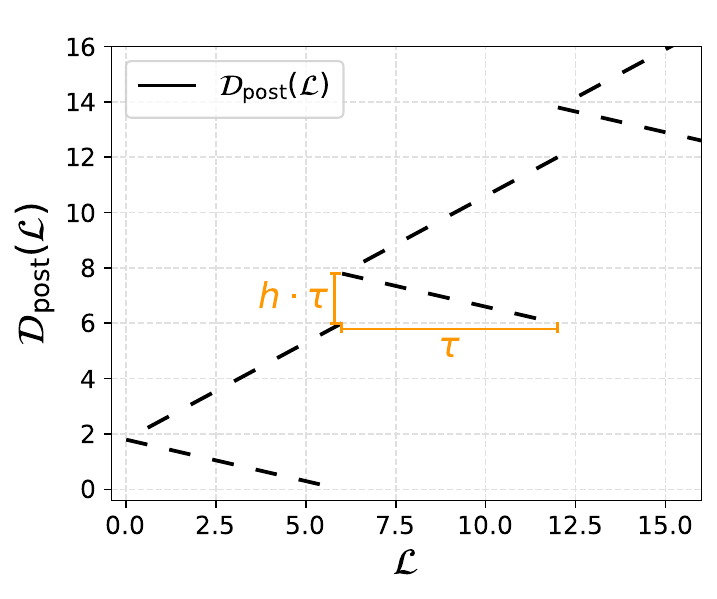}
  \caption{$\mathcal{D}_{\text{post}}$ of $(\tau,h,S)$-DLD with $\tau = 8$, $h = 0.3$, and $S = (0.1,0.2) \cup (0.3,0.4) \cup \dots \cup (0.9,1)$. Although the image may appear as dashed lines, it is actually composed of multiple solid lines.}
  \label{fig:dld}
\end{figure}

Given $\tau$, $h$, and $S$, the function $\mathcal{D}_{\text{post}}$ defined by DLD introduces ambiguity in the relationship between observed loss and the quality of AE candidates.  
As illustrated in Figure~\ref{fig:dld}, each “\(<\)” interval consists of two branches: if a sample's true loss maps to the upper branch ($\mathcal{L}_{\text{high}}$), a better sample (smaller true loss) yields a smaller observed loss; but if it maps to the lower branch ($\mathcal{L}_{\text{low}}$), a better sample produces a larger observed loss.  
Since the attacker cannot determine whether a sample corresponds to $\mathcal{L}_{\text{high}}$ or $\mathcal{L}_{\text{low}}$, deciding whether to accept a higher or lower loss becomes unreliable.
Formally, $\mathcal{D}_{\text{post}}$ satisfies:
\begin{itemize}
\item For any $\mathcal{L}_1, \mathcal{L}_2$ with $\mathcal{L}_{\text{bias}}(\mathcal{L}_1) = \mathcal{L}_{\text{bias}}(\mathcal{L}_2)$, we have $\mathcal{L}_{\text{low}}(\mathcal{L}_1) \le \mathcal{L}_{\text{high}}(\mathcal{L}_2)$.
\item If $\mathcal{L}_{\text{bias}}(\mathcal{L}_1) \!=\! \mathcal{L}_{\text{bias}}(\mathcal{L}_2)$, $\mathcal{L}_1 \!<\! \mathcal{L}_2$, and $\mathcal{D}_{\text{post}}(\mathcal{L}_2) \!=\! \mathcal{L}_{\text{low}}(\mathcal{L}_2)$, then $\mathcal{D}_{\text{post}}(\mathcal{L}_1) \!>\! \mathcal{D}_{\text{post}}(\mathcal{L}_2)$.
\item If $\mathcal{L}_{\text{bias}}(\mathcal{L}_1) \!=\! \mathcal{L}_{\text{bias}}(\mathcal{L}_2)$, $\mathcal{L}_1 \!<\! \mathcal{L}_2$, and $\mathcal{D}_{\text{post}}(\mathcal{L}_2) \!=\! \mathcal{L}_{\text{high}}(\mathcal{L}_2)$, then $\mathcal{D}_{\text{post}}(\mathcal{L}_1) \!<\! \mathcal{D}_{\text{post}}(\mathcal{L}_2)$.
\end{itemize}
These properties ensure discontinuities that simultaneously mislead both minimization and maximization.

During minimization, suppose the attacker holds a current best sample $\mathbf{x}^k$ with true loss $\mathcal{L}^k$, where $\mathcal{D}_{\text{post}}(\mathcal{L}^k) = \mathcal{L}_{\text{high}}(\mathcal{L}^k)$.  
If a new candidate $\mathbf{x}^{\text{new}}$ has a smaller true loss $\mathcal{L}^{\text{new}} < \mathcal{L}^k$ but maps to $\mathcal{L}_{\text{low}}(\mathcal{L}^{\text{new}})$, it will be accepted, yet the second property ensures that any smaller true loss in the same “\(<\)” interval will be assigned a higher observed loss.  
This creates a local “trap” that misleads further minimization.

The attacker can avoid the trap only in two ways:  
(1) generating a much better $\mathbf{x}^{\text{new}}$ with significantly lower loss to jump into the next interval, or  
(2) landing in regions where $\mathcal{D}_{\text{post}}(\mathcal{L}^{\text{new}}) = \mathcal{L}_{\text{high}}(\mathcal{L}^{\text{new}})$.  
The first is rare, while the second is also difficult, with only about a $1/2$ chance per trial for typical $S$, and an overall probability that becomes very
small over successive iterations.
Eventually, with high probability, the observed loss stabilizes near $\mathcal{L}_{\text{bias}}(\mathcal{L}^k)$. Switching to maximization merely flips the roles of $\mathcal{L}_{\text{high}}$ and $\mathcal{L}_{\text{low}}$, highlighting DLD's symmetric disruption of both optimization directions.

\subsection{Theoretical Analysis of DLD}

To formally understand the defensive behavior of DLD, we consider the following three assumptions to characterize the properties of the model, the clean input data, and the attacker's optimizer, respectively.

\begin{assumption}
\label{ass:robust}
The victim model $f$ is $(r,\epsilon/2)$-globally-robust 
within the region $\mathcal{N}(\mathbf{x}_0,\epsilon_\text{n})$.
\end{assumption}

\begin{assumption}
\label{ass:x0}
$\mathcal{L}_{\text{real}}(\mathbf{x}_0) \ge \tau$.

\end{assumption}

\begin{assumption}
\label{ass:iteration}
For all $\mathbf{x}$, $\Vert G_{(\mathbf{x}_0,\epsilon_\text{n})}(\mathbf{x}) - \mathbf{x} \Vert \le r$.
\end{assumption}

For theoretical tractability, we introduce a randomized version of DLD. This variant approximates the deterministic $(\tau,h,S)$-DLD while allowing us to model $\mathcal{D}_{\text{post}}$ outputs as i.i.d. Bernoulli trials.

\begin{definition}
\label{def:rand}
A $(\tau, h, p)$-random-DLD is a post-processing defense module where $\mathcal{D}_{\text{post}}(\mathcal{L})$ is randomly chosen as follows:
\begin{equation*}
\begin{aligned}
\mathrm{P}\big(\mathcal{D}_{\text{post}}(\mathcal{L}) = \mathcal{L}_{\text{high}}(\mathcal{L})\big) &= p,
\\
\mathrm{P}\big(\mathcal{D}_{\text{post}}(\mathcal{L}) = \mathcal{L}_{\text{low}}(\mathcal{L})\big) &= 1-p,
\end{aligned}
\end{equation*}
where $\mathcal{L}_{\text{high}}$ and $\mathcal{L}_{\text{low}}$ are defined in Eqs.~(\ref{eq:high}) and (\ref{eq:low}), respectively.
\end{definition}

Note that $(\tau, h, p)$-random-DLD and $(\tau, h, S)$-DLD share the same good properties and are very similar in practice; only the random version is easier to bypass, and is therefore introduced solely for theoretical analysis (see Appendix \ref{apd:rand-DLD} for details).

We now evaluate DLD's robustness against standard score-based black-box attacks. The following theorem considers a strong adversary with an \emph{infinite query budget} using the standard-SQA procedure under the margin loss defined in Eq.~(\ref{eq:margin_loss_1}), and shows that the success probability remains exponentially small. The proof is in Appendix \ref{apd:proof}.

\begin{theorem}
\label{thm:smallprob}
Under assumptions \ref{ass:robust}, \ref{ass:x0}, and \ref{ass:iteration}, the success probability of standard-SQA$(\mathbf{x}_0,G_{(\mathbf{x}_0,\epsilon_\text{n})})$ against a $(\tau,h,p)$-random-DLD defended $f$ with infinite queries is at most
\begin{equation}\notag
p^{\left\lfloor \frac{\tau - \epsilon}{\epsilon} \right\rfloor}.
\end{equation}
\end{theorem}

This confirms the earlier intuition: even with unlimited queries, breaking DLD via standard-SQA is exponentially unlikely. A larger $\tau$ further increases the attack difficulty.

\begin{remark}
Assumption~\ref{ass:robust} is reasonable, as prior work~\citep{katz2017towards,katz2017reluplex} provided a method to verify it. 
By contrast, the assumptions used in RND~\citep{RND}—requiring the model to be Lipschitz continuous and its derivative to also be Lipschitz continuous—are overly restrictive in practice. Even then, their analysis yields little insight: lower bounds were needed, but only upper bounds were obtained.
\end{remark}

The next theorem quantifies how DLD disrupts reverse-tactic-SQA under the same margin loss defined in Eq.~(\ref{eq:margin_loss_1}) by forcing frequent direction switches. The proof is in Appendix \ref{apd:proof}.

\begin{theorem}
\label{thm:reverse}
Under assumptions \ref{ass:robust}, \ref{ass:x0}, and \ref{ass:iteration}, the expected number of reversals (i.e., the binary flag $reverse$ flips from \texttt{True} to \texttt{False}, or vice versa) in reverse-tactic-SQA$(\mathbf{x}_0, G_{(\mathbf{x}_0,\epsilon_\text{n})}, t)$ against $(\tau, h, p)$-random-DLD defended model $f$ is at least
\begin{equation}\notag
2\left\lfloor\frac{\tau - 2\epsilon}{\epsilon}\right\rfloor \cdot p(1 - p).
\end{equation}
\end{theorem}

This result highlights the inherent limitation of Algorithm \ref{alg:reverse} under DLD. Empirically (see Section~\ref{sec:results}), the attack’s effectiveness is even lower than this bound, since the proof assumes an optimally capable attacker.

\subsection{DLD under Adaptive Attack}

In the presence of defensive mechanisms like DLD, attackers should take into account the potential influence of the defense when designing their strategy. A practical heuristic approach is to occasionally accept worse candidates in order to escape local traps during the optimization process, rather than strictly following the descent direction. This idea is illustrated in Algorithm~\ref{alg:rand}, where the function $Prob$ determines the acceptance probability of a non-improving candidate in each iteration. This reflects an adaptive mindset that accounts for defensive interference.

In this paper, we, as defenders, also adopt an adaptive mindset, assuming that attackers may employ heuristic tactics. Specifically, we adopt a simulated annealing–style acceptance strategy as a representative tactic. This method makes minimal changes to a standard SQA while adding the flexibility needed to escape local traps. As shown in Section~\ref{sec:results}, 
even under this adaptive tactic—where DLD exhibits its worst-case performance—DLD remains as effective as the state-of-the-art AAA.

\begin{algorithm}[tb]
\caption{randomized-tactic-\\SQA$(\mathbf{x}_0,G,Prob)$}
\label{alg:rand}
\Input{clean sample $\mathbf{x}_0$, sample generator $G$, probability generator $Prob$}
\Output{AE candidate $\mathbf{x}$}

$y_0 \gets f(\mathbf{x}_0)$, $\mathbf{x} \gets \mathbf{x}_0$\;

\While{$\mathcal{L}(\mathbf{x},y_0)>0$ \textbf{and} not out of query budget}{
    $\mathbf{x}_{\text{new}} \gets G(\mathbf{x})$\;
    
    \If{$\mathcal{L}(\mathbf{x}_{\text{new}},y_0) < \mathcal{L}(\mathbf{x},y_0)$}{
        $\mathbf{x} \gets \mathbf{x}_{\text{new}}$\;
    }
    \Else{
        with probability $Prob$, $\mathbf{x} \gets \mathbf{x}_{\text{new}}$\;
    }
}

\KwRet{$\mathbf{x}$}
\end{algorithm}

\section{Experiments}

We assess the effectiveness of DLD on ImageNet against adaptive attacks and compare it with existing plug-and-play defenses. Our experiments demonstrate DLD's superior robustness and analyzing the impact of its key parameters. Code is available at https://github.com/fyzdalao/Dashed-Line-Defense.

\subsection{Setup}

We compare DLD against two plug-and-play baselines: RND, an state-of-the-art pre-processing defense that injects noise into the model input, and AAA, a post-processing defense previously described.
For RND, we adopt two noise levels, $\nu = 0.01$ and $\nu = 0.02$, following its original settings. For AAA, we implement both its linear and sinusoidal variants with original default parameters ($\tau = 6$, $\alpha = 0.7$). 
For DLD, we evaluate two variants: the deterministic $(\tau, h, S)$-DLD, where $S = (0, 0.02] \cup (0.04, 0.06] \cup \dots \cup (0.96, 0.98]$, and the randomized $(\tau, h, p)$-random-DLD  with $p = |S| = 0.5$. Both use $\tau = 6$ and $h = 0.3$.

In line with prior work, the evaluation uses 1,000 correctly classified ImageNet samples, one randomly selected from each class. We use PyTorch-pretrained WideResNet-50, RegNet-Y 1.6GF, and MaxViT-T, which we observed to exhibit increasing robustness in this order. AEs are generated by Square Attack~\citep{Square} with an $\ell_\infty$ noise budget $\epsilon_\text{n} = 0.05$ and query budget $n=2500$.

In this work, we exclusively use Square Attack for generating AEs, as the recent benchmark BlackboxBench \citep{zheng2025blackboxbench} identifies it as the strongest black-box SQA available, significantly outperforming other methods.
Echoing this, we observed in initial tests that commonly used alternatives such as Bandits Attack \citep{ilyas2019prior} and Sign Hunter \citep{al2020sign} were not strong enough to yield meaningful experimental results in our setting.

To evaluate robustness against adaptive tactics, we modify Square Attack in four ways. “Standard” uses the original update rule (Algorithm~\ref{alg:standard}); “reverse” flips the update direction (Algorithm~\ref{alg:reverse}); “explore” introduces fixed-probability stochastic exploration (Algorithm~\ref{alg:rand}); and “SA” applies simulated annealing with a time-decaying acceptance probability (Algorithm~\ref{alg:rand}); see Appendix~\ref{apd:parameterForAdaptive} for settings.

\subsection{Numerical Results}
\label{sec:results}

Table~\ref{table:big} presents the test accuracy of various defense methods under different attack tactics. Each row corresponds to an attack tactic, and each column represents a defense method applied to a given model. For each defense, the lowest accuracy across all attack tactics is underlined, reflecting the worst-case performance under adaptive attacks.

We first compare the deterministic $(6, 0.3, S)$-DLD with the randomized $(6, 0.3, 0.5)$-random-DLD, shown in the two rightmost columns of Table~\ref{table:big}. Across all attack tactics, the two variants achieve nearly identical performance, suggesting similar underlying mechanisms: in the randomized version, the probability that $\mathcal{D}_{\text{post}} = \mathcal{L}_{\text{high}}$ is explicitly set to $p$, whereas in the deterministic version this probability is not strictly $p$ but is approximately so in practice. Therefore, theoretical conclusions derived for the randomized version also apply to its deterministic counterpart.

Then, from a game-theoretic perspective, the interaction between attack tactics and defense methods can be viewed as a two-player game, where the underlined accuracy in the DLD column corresponds to the global Nash equilibrium—i.e., the result when both attacker and defender adopt optimal strategies. The fact that this equilibrium lies with DLD, with its underlined value substantially exceeding those of the baselines, indicates that DLD provides stronger defense against adaptive attacks.

Appendix~\ref{sec:asr} presents the attack success rate (ASR) trends during attack iterations, and Appendix~\ref{sec:l2} provides results with $\ell_2$-bounded noise budgets.

\begin{table*}[t]
  \renewcommand{\arraystretch}{0.9}
  \caption{The test accuracy under attacks and defenses (\%)}
  \label{table:big}
  \centering
  \begin{tabular}{c c c c c c c c c c}
\toprule
\multirow{3}{*}{Model} & \multicolumn{2}{c}{Attack} & \multicolumn{7}{c}{Defense} \\ \cmidrule(lr){2-3} \cmidrule(lr){4-10}
 & \multirow{2}{*}{generator $G$} & \multirow{2}{*}{tactic} & \multirow{2}{*}{None} & \multicolumn{2}{c}{RND} & \multicolumn{2}{c}{AAA} & \multicolumn{2}{c}{DLD} \\ 
 &  &  &  & 0.01 & 0.02 & sine & linear & rand & determine \\ \midrule

\multirow{5}{*}{WideResNet} 
 & \multicolumn{2}{c}{~ ~ ~ ~None} & 100 & 97.4 & 95.8 & 100 & 100 & 100 & 100 \\ 
 & \multirow{4}{*}{Square Attack} & standard & 1.6 & \underline{46.4} & 53.3 & 52.4 & 73.4 & 70.2 & 72.3 \\ 
 &  & explore & 36.4 & 52.8 & \underline{53.1} & 55.8 & 70.8 & 64.8 & 64.8 \\ 
 &  & SA & 58.1 & 56.7 & 56.2 & 55.9 & 61.9 & \underline{61.3} & \underline{60.9} \\ 
 &  & reverse & 5.9 & 55.5 & 55.7 & \underline{36.8} & \underline{22.2} & 71.9 & 71.9 \\ \midrule

\multirow{5}{*}{RegNet-Y} 
 & \multicolumn{2}{c}{~ ~ ~ ~None} & 100 & 97.7 & 95.7 & 100 & 100 & 100 & 100 \\ 
 & \multirow{4}{*}{Square Attack} & standard & 3.1 & \underline{49.8} & \underline{53.8} & 61.7 & 79.9 & 76.9 & 76.4 \\ 
 &  & explore & 32.4 & 56.7 & 56.3 & 62.1 & 76.7 & 69.2 & 69.4 \\ 
 &  & SA & 61.5 & 62.5 & 58.4 & 62.6 & 68.5 & \underline{66.6} & \underline{66.3} \\ 
 &  & reverse & 6.9 & 57.6 & 58.2 & \underline{41.5} & \underline{17.5} & 78.6 & 77.7 \\ \midrule

\multirow{5}{*}{MaxViT-T} 
 & \multicolumn{2}{c}{~ ~ ~ ~None} & 100 & 98.6 & 97.5 & 100 & 100 & 100 & 100 \\ 
 & \multirow{4}{*}{Square Attack} & standard & 15.1 & \underline{71.3} & \underline{73.0} & 70.5 & 91.9 & 90.3 & 91.0 \\ 
 &  & explore & 56.7 & 76.1 & 73.6 & 78.7 & 89.6 & 84.6 & 84.2 \\ 
 &  & SA & 78.0 & 78.4 & 76.1 & 77.3 & 82.2 & \underline{82.2} & \underline{81.4} \\ 
 &  & reverse & 39.1 & 77.0 & 76.3 & \underline{56.3} & \underline{40.6} & 90.8 & 90.9 \\ \bottomrule
  \end{tabular}
\end{table*}

\subsection{Hyper-Parameters of DLD}
\label{sec:hyper-param}

\textbf{The choice of $S$} controls whether $\mathcal{D}_{\text{post}}$ returns $\mathcal{L}_{\text{high}}$ or $\mathcal{L}_{\text{low}}$. We define $S$ as a union of small intervals within $(0,1)$: $S = \bigcup_k ((k{-}ratio)\cdot step,\ k\cdot step) \cap (0,1)$, where $step \in (0,1]$ and $ratio \in [0,1]$. When $1/step$ is an integer, this simplifies to $S = \bigcup_k ((k{-}ratio)\cdot step,\ k\cdot step)$ with $|S| = ratio$.

A smaller $|S|$ improves defense by increasing loss ambiguity (Theorem~\ref{thm:smallprob}) but risks making $\mathcal{D}_{\text{post}}$ too continuous, aiding Algorithm~\ref{alg:reverse}. Conversely, larger $|S|$ better preserves the model output but weakens defense due to reduced distortion.
We evaluate this trade-off by running Square Attack on $(6, 0.3, S)$-DLD with $step = 0.04$ and $ratio$ from $0.005$ to $0.99$, using a larger noise budget $\epsilon_\text{n} = 0.07$ to amplify effects. Figure~\ref{fig:ratio} shows that defense performance remains strong at large $ratio$, degrading only when $ratio > 0.9$. At the other end, the reverse tactic is effective only when $ratio < 0.015$.
These results suggest that while $ratio = 0.5$ is a safe default, DLD remains effective across a broad $ratio$ range, enabling flexible trade-offs between robustness and output fidelity.

\begin{figure}[t]
  \centering
  \begin{subfigure}[b]{0.23\textwidth}
    \includegraphics[width=\textwidth]{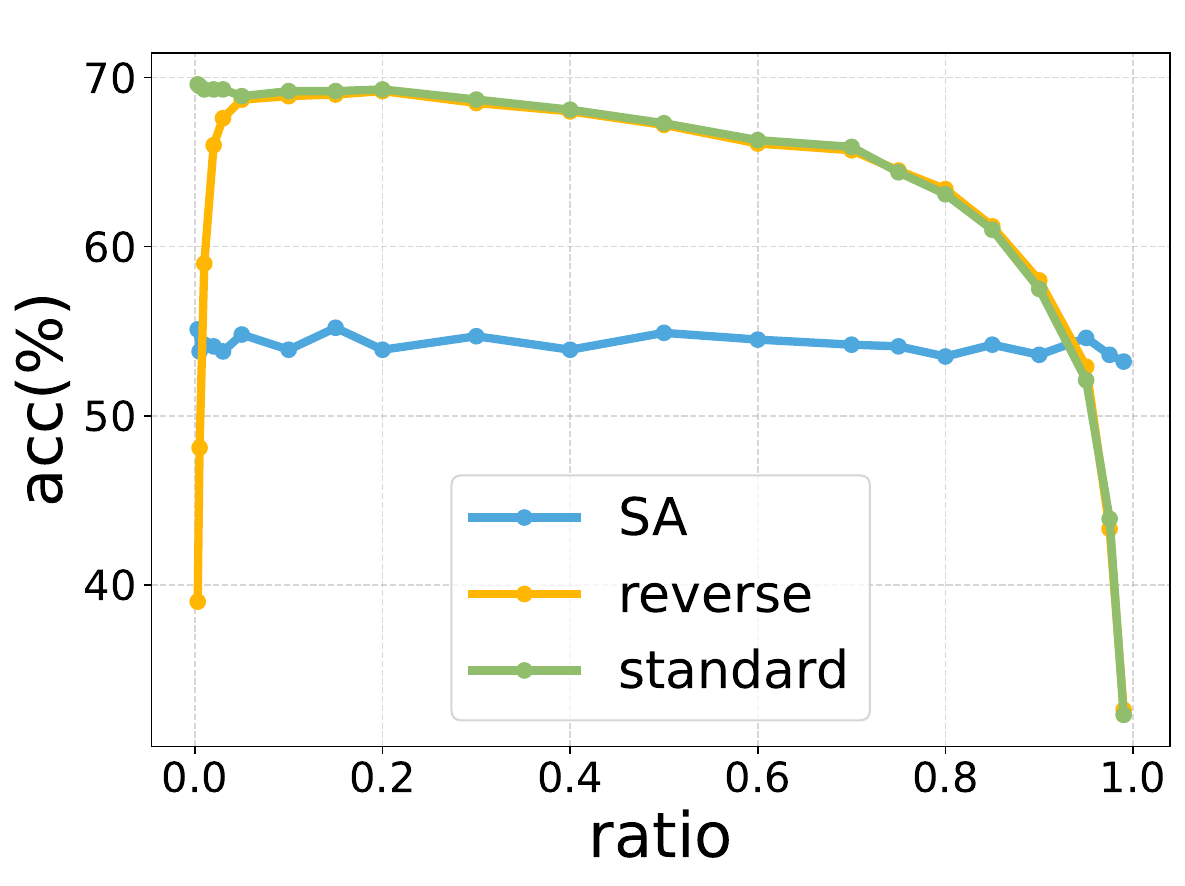}
    \caption{Accuracy vs. ratio}
    \label{fig:ratio}
  \end{subfigure}
  \hfill
  \begin{subfigure}[b]{0.23\textwidth}
    \includegraphics[width=\textwidth]{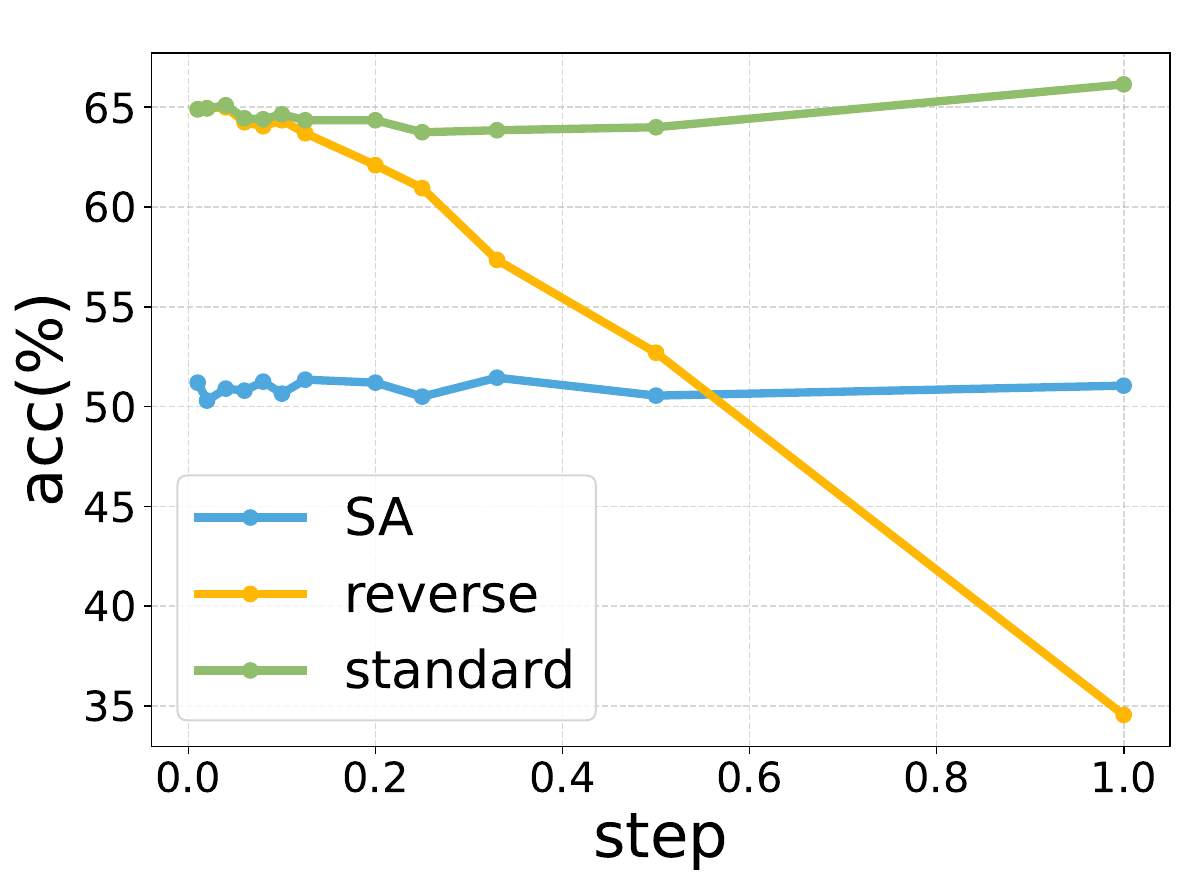}
    \caption{Accuracy vs. step}
    \label{fig:step}
  \end{subfigure}
  \begin{subfigure}[b]{0.34\textwidth}
    \includegraphics[width=\textwidth]{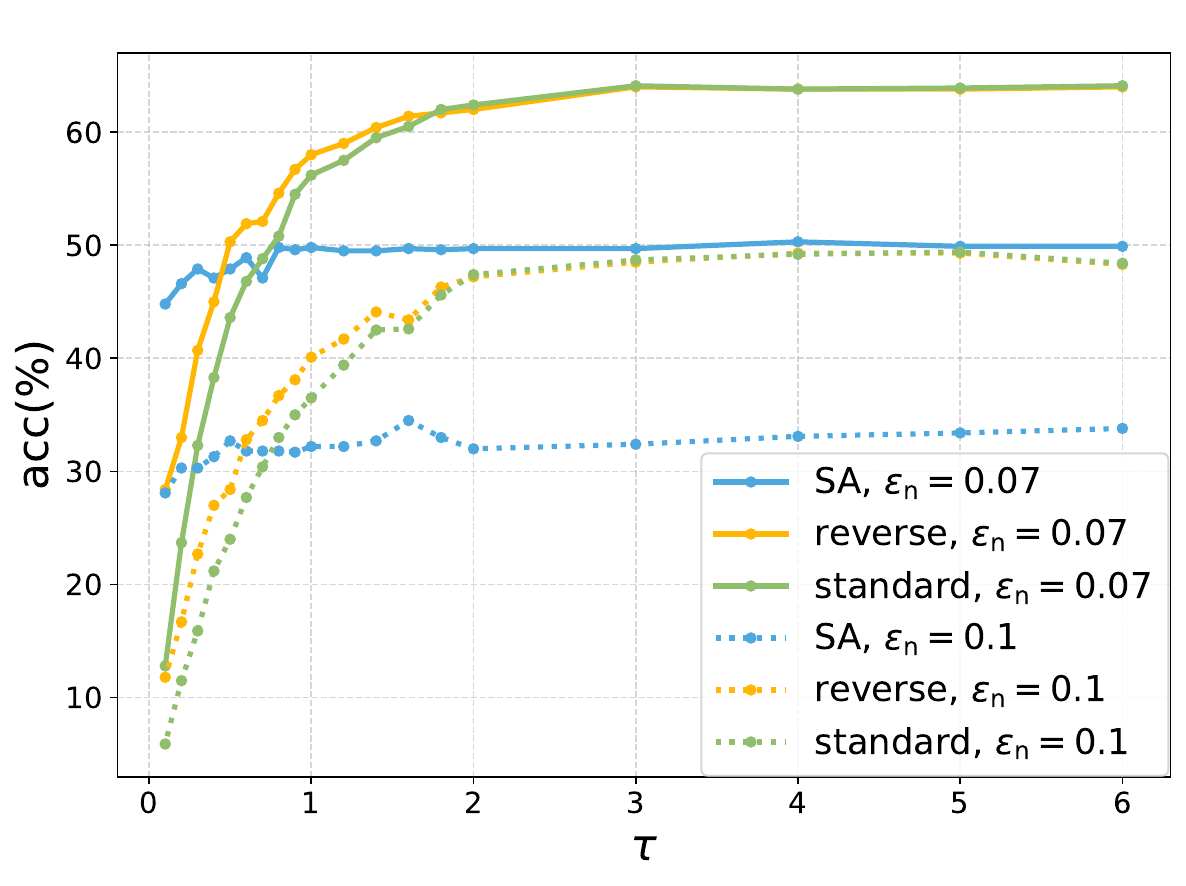}
    \caption{Accuracy vs. $\tau$}
    \label{fig:tau}
  \end{subfigure}
  \caption{Accuracy under Square Attack with different DLD parameters.}
  \label{fig:ratio-step-tau}
\end{figure}

We also assess the effect of $step$ by fixing $ratio = 0.5$ and varying $step$ from $0.01$ to $1$ (Figure~\ref{fig:step}).
As $step$ increases, DLD increasingly resembles AAA, resulting in a minor performance gain against the standard attack but a significant degradation under the reverse tactic.
Nonetheless, with moderate $step$, performance under reverse tactic remains above that of the SA tactic, leaving the attack-defense Nash equilibrium unchanged.

Figure \ref{fig:ratio-step-tau} shows that variations in $ratio$ and $step$ have little impact on performance under SA tactic.
Note that both lower $ratio$ and larger $step$ make DLD more like AAA-linear. 
Thus, DLD and AAA-linear should exhibit similar performance under SA tactic, which is indeed consistent with Section~\ref{sec:results}.

\textbf{The choice of $\tau$}, as discussed in Section~\ref{sec:DLD}, determines the extent to which DLD alters the output scores. 
We vary $\tau$ while fixing $step = 0.04$ and $ratio = 0.5$ under three Square Attack variants (Figure~\ref{fig:tau}), using $n = 2500$ and $\epsilon_\text{n} \in \{0.07, 0.1\}$.
Smaller $\tau$ introduces less output distortion but weakens defense, especially when the noise budget $\epsilon_\text{n}$ is large. In contrast, larger $\tau$ consistently strengthens defense without abrupt drops, making $\tau$ easy to tune: when exact prediction scores are less critical, a larger $\tau$ can be safely chosen to enhance protection.

\section{Conclusion}
In this work, we revealed that the current state-of-the-art runtime defense against SQA remains vulnerable to adaptive tactics.
To address this gap, we introduced Dashed Line Defense (DLD), a plug-and-play post-processing method with a non-continuous loss mapping to confuse attack optimization. Our theoretical analysis and ImageNet experiments demonstrate that DLD provides strong defense and outperforms prior approaches even under worst-case adaptive scenarios. 
We believe this work highlights the need for defenses that explicitly account for attacker adaptivity, while also offering insights for developing more adaptive and diverse attack tactics, inspiring future research on both sides of the attack–defense landscape.

\bibliographystyle{unsrtnat} 
\bibliography{refs}

@inproceedings{guo2024datadebugging,
    author={Zizheng Guo and Jun Wu and Pengyu Chen and Yanzhang Fu and Dongjing Miao},
    title={Data Debugging is {NP-hard} for Classifiers Trained with {SGD}},
    booktitle ={International Computing and Combinatorics Conference (COCOON)},
    year = {2025}
}

@article{AAA,
  title={Adversarial attack on attackers: Post-process to mitigate black-box score-based query attacks},
  author={Chen, Sizhe and Huang, Zhehao and Tao, Qinghua and Wu, Yingwen and Xie, Cihang and Huang, Xiaolin},
  journal={Advances in Neural Information Processing Systems (NeurIPS)},
  volume={35},
  pages={14929--14943},
  year={2022}
}

@article{RND,
  title={Random noise defense against query-based black-box attacks},
  author={Qin, Zeyu and Fan, Yanbo and Zha, Hongyuan and Wu, Baoyuan},
  journal={Advances in Neural Information Processing Systems (NeurIPS)},
  volume={34},
  pages={7650--7663},
  year={2021}
}

@inproceedings{Square,
  title={Square attack: A query-efficient black-box adversarial attack via random search},
  author={Andriushchenko, Maksym and Croce, Francesco and Flammarion, Nicolas and Hein, Matthias},
  booktitle={European Conference on Computer Vision (ECCV)},
  pages={484--501},
  year={2020},
  organization={Springer}
}

@article{katz2017towards,
   title={Towards Proving the Adversarial Robustness of Deep Neural Networks},
   volume={257},
   ISSN={2075-2180},
   url={http://dx.doi.org/10.4204/EPTCS.257.3},
   DOI={10.4204/eptcs.257.3},
   journal={Electronic Proceedings in Theoretical Computer Science},
   publisher={Open Publishing Association},
   author={Katz, Guy and Barrett, Clark and Dill, David L. and Julian, Kyle and Kochenderfer, Mykel J.},
   year={2017},
   month=sep, pages={19–26} }

@inproceedings{katz2017reluplex,
  title={Reluplex: An efficient {SMT} solver for verifying deep neural networks},
  author={Katz, Guy and Barrett, Clark and Dill, David L and Julian, Kyle and Kochenderfer, Mykel J},
  booktitle={International Conference on Computer Aided Verification},
  pages={97--117},
  year={2017},
  organization={Springer}
}

@article{tramer2020adaptive,
  title={On adaptive attacks to adversarial example defenses},
  author={Tramer, Florian and Carlini, Nicholas and Brendel, Wieland and Madry, Aleksander},
  journal={Advances in Neural Information Processing Systems (NeurIPS)},
  volume={33},
  pages={1633--1645},
  year={2020}
}

@inproceedings{chen2017zoo,
  title={{ZOO}: Zeroth order optimization based black-box attacks to deep neural networks without training substitute models},
  author={Chen, Pin-Yu and Zhang, Huan and Sharma, Yash and Yi, Jinfeng and Hsieh, Cho-Jui},
  booktitle={ACM Workshop on Artificial Intelligence and Security},
  pages={15--26},
  year={2017}
}

@inproceedings{szegedy2014intriguing,
  title={Intriguing Properties of Neural Networks},
  author={Szegedy, Christian and Zaremba, Wojciech and Sutskever, Ilya and Bruna, Joan and Erhan, Dumitru and Goodfellow, Ian and Fergus, Rob},
  booktitle={International Conference on Learning Representations (ICLR)},
  year={2014}
}

@misc{goodfellow2015explaining,
      title={Explaining and Harnessing Adversarial Examples}, 
      author={Ian J. Goodfellow and Jonathon Shlens and Christian Szegedy},
      year={2015},
      eprint={1412.6572},
      archivePrefix={arXiv},
      primaryClass={stat.ML},
      url={https://arxiv.org/abs/1412.6572}, 
}

@inproceedings{ilyas2018black,
  title={Black-box adversarial attacks with limited queries and information},
  author={Ilyas, Andrew and Engstrom, Logan and Athalye, Anish and Lin, Jessy},
  booktitle={International Conference on Machine Learning (ICML)},
  pages={2137--2146},
  year={2018},
  organization={PMLR}
}

@inproceedings{bhagoji2018practical,
  title={Practical black-box attacks on deep neural networks using efficient query mechanisms},
  author={Bhagoji, Arjun Nitin and He, Warren and Li, Bo and Song, Dawn},
  booktitle={European Conference on Computer Vision (ECCV)},
  pages={154--169},
  year={2018}
}

@inproceedings{ilyas2019prior,
  title={Prior Convictions: Black-box Adversarial Attacks with Bandits and Priors},
  author={Ilyas, Andrew and Engstrom, Logan and Madry, Aleksander},
  booktitle={International Conference on Learning Representations (ICLR)},
  year={2019}
}

@inproceedings{al2020sign,
  title={Sign bits are all you need for black-box attacks},
  author={Al-Dujaili, Abdullah and O'Reilly, Una-May},
  booktitle={International Conference on Learning Representations (ICLR)},
  year={2020}
}

@inproceedings{alzantot2019genattack,
  title={Genattack: Practical black-box attacks with gradient-free optimization},
  author={Alzantot, Moustafa and Sharma, Yash and Chakraborty, Supriyo and Zhang, Huan and Hsieh, Cho-Jui and Srivastava, Mani B},
  booktitle={Genetic and Evolutionary Computation Conference},
  pages={1111--1119},
  year={2019}
}

@inproceedings{carlini2017towards,
  title={Towards evaluating the robustness of neural networks},
  author={Carlini, Nicholas and Wagner, David},
  booktitle={IEEE Symposium on Security and Privacy (SP)},
  pages={39--57},
  year={2017},
  organization={Ieee}
}

@inproceedings{papernot2017practical,
  title={Practical black-box attacks against machine learning},
  author={Papernot, Nicolas and McDaniel, Patrick and Goodfellow, Ian and Jha, Somesh and Celik, Z Berkay and Swami, Ananthram},
  booktitle={ACM Asia Conference on Computer and Communications Security},
  pages={506--519},
  year={2017}
}

@inproceedings{athalye2018obfuscated,
  title={Obfuscated gradients give a false sense of security: Circumventing defenses to adversarial examples},
  author={Athalye, Anish and Carlini, Nicholas and Wagner, David},
  booktitle={International Conference on Machine Learning (ICML)},
  pages={274--283},
  year={2018}
}

@article{yao2021automated,
  title={Automated discovery of adaptive attacks on adversarial defenses},
  author={Yao, Chengyuan and Bielik, Pavol and Tsankov, Petar and Vechev, Martin},
  journal={Advances in Neural Information Processing Systems (NeurIPS)},
  volume={34},
  pages={26858--26870},
  year={2021}
}

@inproceedings{madry2018towards,
  title={Towards Deep Learning Models Resistant to Adversarial Attacks},
  author={Madry, Aleksander and Makelov, Aleksandar and Schmidt, Ludwig and Tsipras, Dimitris and Vladu, Adrian},
  booktitle={International Conference on Learning Representations (ICLR)},
  year={2018}
}

@inproceedings{tramer2018ensemble,
  title={Ensemble Adversarial Training: Attacks and Defenses},
  author={Tramer, Florian and Kurakin, Alexey and Papernot, Nicolas and Goodfellow, Ian and Boneh, Dan and McDaniel, Patrick},
  booktitle={International Conference on Learning Representations (ICLR)},
  year={2018}
}

@inproceedings{xie2019improving,
  title={Improving transferability of adversarial examples with input diversity},
  author={Xie, Cihang and Zhang, Zhishuai and Zhou, Yuyin and Bai, Song and Wang, Jianyu and Ren, Zhou and Yuille, Alan L},
  booktitle={IEEE/CVF Conference on Computer Vision and Pattern Recognition (CVPR)},
  pages={2730--2739},
  year={2019}
}

@inproceedings{byun2022effectiveness,
  title={On the effectiveness of small input noise for defending against query-based black-box attacks},
  author={Byun, Junyoung and Go, Hyojun and Kim, Changick},
  booktitle={IEEE/CVF Winter Conference on Applications of Computer Vision},
  pages={3051--3060},
  year={2022}
}

@inproceedings{chen2020stateful,
  title={Stateful detection of black-box adversarial attacks},
  author={Chen, Steven and Carlini, Nicholas and Wagner, David},
  booktitle={ACM Workshop on Security and Privacy on Artificial Intelligence},
  pages={30--39},
  year={2020}
}

@inproceedings{guo2019simple,
  title={Simple black-box adversarial attacks},
  author={Guo, Chuan and Gardner, Jacob and You, Yurong and Wilson, Andrew Gordon and Weinberger, Kilian},
  booktitle={International Conference on Machine Learning (ICML)},
  pages={2484--2493},
  year={2019},
  organization={PMLR}
}

@inproceedings{xie2019feature,
  title={Feature Denoising for Improving Adversarial Robustness},
  author={Xie, Cihang and Wu, Yuxin and van der Maaten, Laurens and Yuille, Alan and He, Kaiming},
  booktitle={IEEE/CVF Conference on Computer Vision and Pattern Recognition (CVPR)},
  year={2019},
  pages={501--509}
}

@inproceedings{qin2021random,
  title={Random Weight Noise for Improved Robustness: A Case Study in Visual Classification},
  author={Qin, Chengyu and Wang, Xiaoyu and Zhang, Xinyang and Yu, Huan and Chen, Pin-Yu},
  booktitle={IEEE International Conference on Computer Vision (ICCV)},
  year={2021},
  pages={6414--6424}
}

@inproceedings{cohen2019certified,
  title={Certified Adversarial Robustness via Randomized Smoothing},
  author={Cohen, Jeremy M and Rosenfeld, Elan and Kolter, Zico},
  booktitle={International Conference on Machine Learning (ICML)},
  year={2019},
  pages={1310--1320}
}

@inproceedings{shafahi2019adversarial,
  title={Adversarial Training for Free!},
  author={Shafahi, Ali and Najibi, Mahyar and Ghiasi, Amin and Xu, Zheng and Dickerson, John and Studer, Christoph and Davis, Larry S and Goldstein, Tom},
  booktitle={Advances in Neural Information Processing Systems (NeurIPS)},
  year={2019}
}

@inproceedings{schmidt2018adversarially,
  title={Adversarially Robust Generalization Requires More Data},
  author={Schmidt, Ludwig and Santurkar, Shibani and Tsipras, Dimitris and Talwar, Kunal and Madry, Aleksander},
  booktitle={Advances in Neural Information Processing Systems (NeurIPS)},
  year={2018}
}

@inproceedings{li2024query,
  title={Query Provenance Analysis: Efficient and Robust Defense against Query-based Black-box Attacks},
  author={Li, Shaofei and Zhang, Ziqi and Jia, Haomin and Guo, Yao and Chen, Xiangqun and Li, Ding},
  booktitle={IEEE Symposium on Security and Privacy (SP)},
  pages={1641--1656},
  year={2025},
  organization={IEEE}
}

@misc{park2025mind,
      title={Mind the Gap: Detecting Black-box Adversarial Attacks in the Making through Query Update Analysis}, 
      author={Jeonghwan Park and Niall McLaughlin and Ihsen Alouani},
      year={2025},
      eprint={2503.02986},
      archivePrefix={arXiv},
      primaryClass={cs.CR},
      url={https://arxiv.org/abs/2503.02986}, 
}

@article{zheng2025blackboxbench,
  title={Blackboxbench: A comprehensive benchmark of black-box adversarial attacks},
  author={Zheng, Meixi and Yan, Xuanchen and Zhu, Zihao and Chen, Hongrui and Wu, Baoyuan},
  journal={IEEE Transactions on Pattern Analysis and Machine Intelligence},
  year={2025},
  publisher={IEEE}
}

@inproceedings{li2020towards,
  title={Towards transferable targeted attack},
  author={Li, Maosen and Deng, Cheng and Li, Tengjiao and Yan, Junchi and Gao, Xinbo and Huang, Heng},
  booktitle={IEEE/CVF Conference on Computer Vision and Pattern Recognition (CVPR)},
  pages={641--649},
  year={2020}
}

@inproceedings{cheng2019query,
  title={Query-efficient hard-label black-box attack: An optimization-based approach},
  author={Cheng, Minhao and Zhang, Huan and Hsieh, Cho Jui and Le, Thong and Chen, Pin Yu and Yi, Jinfeng},
  booktitle={International Conference on Learning Representations (ICLR)},
  year={2019}
}

@inproceedings{chen2020hopskipjumpattack,
  title={Hopskipjumpattack: A query-efficient decision-based attack},
  author={Chen, Jianbo and Jordan, Michael I and Wainwright, Martin J},
  booktitle={IEEE Symposium on Security and Privacy (SP)},
  pages={1277--1294},
  year={2020},
  organization={IEEE}
}

\clearpage
\appendix
\thispagestyle{empty}

% Supplementary material: To improve readability, you must use a single-column format for the supplementary material.
\onecolumn
\aistatstitle{
Appendix}

% Note: You can choose whether the include you appendices as part of the main submission file (here) OR submit them separately as part of the supplementary material. It is the authors' responsibility that any supplementary material does not conflict in content with the main paper (e.g., the separately uploaded additional material is not an updated version of the one appended to the manuscript).

\section{Motivation for Randomized DLD}
\label{apd:rand-DLD}

To facilitate the theoretical analysis of DLD, we rely on a key property: for any input, the output of $\mathcal{D}_{\text{post}}$ equals $\mathcal{L}_{\text{high}}$ with probability $p$ and $\mathcal{L}_{\text{low}}$ with probability $1 - p$. This approximately holds for deterministic DLD if we set $|S| = p$ (see Table~\ref{tab:Sandp}). It does not hold exactly, unless we can guarantee or reasonably assume that the margin loss (input to $\mathcal{D}_{\text{post}}$) is uniformly distributed. Unfortunately, this assumption is not realistic.
Even with an idealized model input distribution, the margin loss distribution is shaped by complex model behavior, which remains poorly understood due to limited progress in deep model interpretability.

\begin{table}[h]
  \centering
  \caption{Proportion of \(\mathcal{D}_{\text{post}} = \mathcal{L}_{\text{high}}\) after applying \((6, 0.3, S)\)-DLD to MaxViT-T predictions on the ImageNet validation set. \(S\) is a union of small, evenly spaced intervals in \((0, 1)\).}
  \label{tab:Sandp}
  \begin{tabular}{c|c}
    \toprule
    \(|S|\) & Proportion of \(\mathcal{D}_{\text{post}} = \mathcal{L}_{\text{high}}\) (\%) \\
    \midrule
    0.1 & 9.906 \\
    0.2 & 19.988 \\
    0.3 & 29.740 \\
    0.4 & 39.864 \\
    0.5 & 49.930 \\
    0.6 & 59.798 \\
    0.7 & 69.908 \\
    0.8 & 80.036 \\
    0.9 & 90.086 \\
    \bottomrule
  \end{tabular}
\end{table}

To formalize this probabilistic behavior, we introduce randomized version of DLD (Definition~\ref{def:rand}), which explicitly enforces the desired probabilities $p$ and $1-p$. Since both versions behave similarly in practice, we argue that theoretical results for randomized DLD can meaningfully approximate those for the deterministic version.

One might ask: why not adopt randomized DLD directly? The answer lies in a critical vulnerability—randomized DLD is susceptible to adaptive attacks. If the attacker knows the defense, they can repeatedly query the same input until two distinct outputs appear. The larger loss must correspond to $\mathcal{L}_{\text{high}}$, which is strictly increasing, and can then be minimized to generate an adversarial example.

The cost of this tactic is not prohibitive. Let $X$ be the number of queries needed to observe two different outputs. This is a variant of the non-uniform coupon collector problem, and we can show $\mathbb{E}[X] = 1 + 1/(p(1-p))$. When $p = 0.5$, the expected number of queries is only $5$. That is, the attacker needs only about $5\times$ more queries to reliably succeed, avoiding the traps introduced by DLD. Using more extreme values of $p$ can increase this cost, but introduces new trade-offs (see Section~\ref{sec:hyper-param}) and still doesn’t eliminate the attack vector.

Although one might mitigate this vulnerability by tricks like varying $p$ across regions, allowing inconsistent outputs for the same input is generally undesirable. Therefore, we adopt the deterministic version of DLD as the default method in this paper.

\section{Proofs}
\label{apd:proof}
\subsection{Proof for Theorem \ref{thm:smallprob}}

\begin{proof}

For convenience, we assume that the generator $G$ is sufficiently strong such that for any input $\mathbf{x}$, we have $\mathcal{L}_{\text{real}}(G(\mathbf{x})) \le \mathcal{L}_{\text{real}}(\mathbf{x})$. If this assumption does not hold, the attack would be even less effective, making the theorem stronger.

In ecah iteration, the standard-SQA$(\mathbf{x}_0,G)$ accepts or rejects the generated example. 
Let $\mathbf{x}_1, \mathbf{x}_2, \dots, \mathbf{x}_m$ denote the sequence of inputs accepted during the attack (i.e., those that lead to observed loss values no greater than their predecessors).

By Assumptions~\ref{ass:robust} and~\ref{ass:iteration}, we have, for all $\mathbf{x}$ and any class $y$,
\[
\vert f_y(\mathbf{x}) - f_y(G(\mathbf{x})) \vert < \frac{\epsilon}{2}.
\]
This implies that the difference in real loss between consecutive accepted examples is bounded:
\[
\vert \mathcal{L}_{\text{real}}(\mathbf{x}) - \mathcal{L}_{\text{real}}(G(\mathbf{x})) \vert < \epsilon.
\]
Therefore, for all $j > 0$, we have
\[
\vert \mathcal{L}_{\text{real}}(\mathbf{x}_{j}) - \mathcal{L}_{\text{real}}(\mathbf{x}_{j-1}) \vert < \epsilon.
\]

If the standard-SQA ever accepts an $\mathbf{x}_{j}$ that $\mathcal{L}_{\text{real}}(\mathbf{x}_{j}) \ge \epsilon$ and $\mathcal{L}_{\text{observed}}(\mathbf{x}_{j}) = \mathcal{D}_{\text{post}}(\mathcal{L}_{\text{real}}(\mathbf{x}_{j})) = \mathcal{L}_{\text{low}}(\mathcal{L}_{\text{real}}(\mathbf{x}_{j}))$, then because of the property of $\mathcal{D}_{\text{post}}$, any $G(\mathbf{x}_{j})$ will meet $\mathcal{L}_{\text{observed}}(G(\mathbf{x}_{j})) > \mathcal{L}_{\text{observed}}(\mathbf{x}_{j})$. Thus, future iterations of the standard-SQA will not accept any examples, making $m=j$. The standard-SQA fails.

The standard-SQA will succeed if an example $\mathbf{x}_i$ is accepted such that $\mathcal{L}_{\text{real}}(\mathbf{x}_i) < \epsilon$ because the query budget is infinite. 
The standard-SQA will succeed only if an example $\mathbf{x}_i$ is accepted such that $\mathcal{L}_{\text{real}}(\mathbf{x}_i) < \epsilon$ because $\forall j >0$, $\vert \mathcal{L}_{\text{real}}(\mathbf{x}_{j}) - \mathcal{L}_{\text{real}}(\mathbf{x}_{j-1}) \vert < \epsilon $.
If such an $\mathbf{x}_i$ is accepted, then for all $j\in \{1,2,...,i-1\}$, $\mathcal{L}_{\text{observed}}(\mathbf{x}_j) = \mathcal{L}_{\text{high}}(\mathcal{L}_{\text{real}}(\mathbf{x}_j))$. 

Note $\forall j >0$, $\vert \mathcal{L}_{\text{real}}(\mathbf{x}_{j}) - \mathcal{L}_{\text{real}}(\mathbf{x}_{j-1}) \vert < \epsilon $, due to assumption \ref{ass:x0}, it will take at least $\lfloor \frac{\tau-\epsilon}{\epsilon} \rfloor +1$ iterations for standard-SQA to obtain $\mathbf{x}_i$. Thus, $i-1 \ge \lfloor \frac{\tau-\epsilon}{\epsilon} \rfloor$. The probability of success
\begin{equation}\notag
\begin{aligned}
& \mathrm{P}\left( \bigcap_{j\in \{1,2,...,i-1\}}\left( \mathcal{L}_{\text{observed}}\left( \mathbf{x}_j \right)= \mathcal{L}_{\text{high}}\left( \mathcal{L}_{\text{real}}\left( \mathbf{x}_j \right) \right) \right) \right) \\
&= \prod_{j=1}^{i-1} \mathrm{P}\left(\mathcal{L}_{\text{observed}}(\mathbf{x}_j) = \mathcal{L}_{\text{high}}(\mathcal{L}_{\text{real}}(\mathbf{x}_j))\right)\\ 
&\le \prod_{j=1}^{\lfloor \frac{\tau-\epsilon}{\epsilon} \rfloor}\mathrm{P}(\mathcal{L}_{\text{observed}}(\mathbf{x}_j) = \mathcal{L}_{\text{high}}(\mathcal{L}_{\text{real}}(\mathbf{x}_j)))
\\
&=\prod_{j=1}^{\lfloor \frac{\tau-\epsilon}{\epsilon} \rfloor} p \\
&= p^{\left\lfloor \frac{\tau - \epsilon}{\epsilon} \right\rfloor}.
\end{aligned}
\end{equation}
\end{proof}

\subsection{Proof for Theorem \ref{thm:reverse}}

\begin{proof}

For convenience, we assume that the generator $G$ is sufficiently strong such that for any input $\mathbf{x}$, we have $\mathcal{L}_{\text{real}}(G(\mathbf{x})) \le \mathcal{L}_{\text{real}}(\mathbf{x})$. If this assumption does not hold, the attack would be even less effective, making the theorem stronger.
Without loss of generality, assume that in every iteration the algorithm always accepts the newly generated example. Arrange the examples accepted by the algorithm in order of occurrence, represented as sequence $\mathbf{x}_1,\mathbf{x}_2,...,\mathbf{x}_m$.

Denote $\mathcal{L}_{\text{observed}}(\mathbf{x}_i) = \mathcal{L}_{\text{high}}(\mathcal{L}_{\text{real}}(\mathbf{x}_i))$ as event $U_i$. According to Definition \ref{def:rand}, for $i\ge 1$, $\mathrm{P}(U_i) = p$, and $\mathrm{P}\left(\overline{U_i}\right) = 1-p$.
For $i\ge 2$, the optimization direction must be reversed in the $i$-th iteration when: 1) $U_i \land \overline{U_{i-1}}$ or 2) $\overline{U_i} \land U_{i-1}$.

The times of conducting direction reverse can be calculated using an indicator
\begin{equation*}
X_i = \left\{
\begin{aligned}
& 1 ,\   \left(U_i \land \overline{U_{i-1}}\right) \lor \left(\overline{U_i} \land U_{i-1} \right)   \\
& 0 , \  otherwise \\
\end{aligned}
\right. 
,i\ge 2
\end{equation*}

We have
\begin{equation*}
\begin{aligned}
\mathrm{E}(X_i) &= \mathrm{E}(X_i | U_i)\mathrm{P}(U_i) + \mathrm{E}\left(X_i | \overline{U_i}\right)\mathrm{P}\left(\overline{U_i}\right)\\
&=\mathrm{P}(X_i=1 | U_i)\mathrm{P}(U_i) + \mathrm{P}\left(X_i=1 | \overline{U_i}\right)\mathrm{P}\left(\overline{U_i}\right)\\
&=\mathrm{P}\left(U_i \land \overline{U_{i-1}}\right) + \mathrm{P}\left(\overline{U_i} \land U_{i-1}\right)\\
&=2p(1-p)
\end{aligned}
\end{equation*}

Let $N=\lfloor \frac{\tau-\epsilon}{\epsilon} \rfloor$. According to the proof of Theorem \ref{thm:smallprob}, it will take at least $N+1$ iterations for the algorithm to succeed. The expectation of total times of conducting direction reverse in $N$ iterations
\begin{equation*}
\begin{aligned}
\mathrm{E}\left(\sum_{i=1}^N X_i\right) 
\ge \mathrm{E}\left(\sum_{i=2}^N X_i\right)
=\sum_{i=2}^N \mathrm{E}(X_i)
=2(N-1)p(1-p)
=2\lfloor \frac{\tau-2\epsilon}{\epsilon} \rfloor  p  (1-p).
\end{aligned}
\end{equation*}

\end{proof}

\section{Parameter Settings for Adaptive Tactics}
\label{apd:parameterForAdaptive}
To examine the potential vulnerabilities of runtime defenses and stress-test our proposed DLD method, we construct several adaptive variants based on Square Attack.

The \textbf{standard} variant follows the original update rule used by Square Attack and similar methods, where each iteration accepts the candidate with a lower loss value. This corresponds to Algorithm~\ref{alg:standard}.

The \textbf{reverse} variant, corresponding to Algorithm~\ref{alg:reverse}, reverses the update direction after the optimization stagnates. In our experiments, the stagnation threshold $t$ is set to 23. This value was empirically chosen to give the attacker the greatest advantage and the defender the greatest challenge.

The \textbf{explore} variant corresponds to Algorithm~\ref{alg:rand}, in which each iteration accepts a worse candidate with fixed probability. We set $Prob = 0.5$, empirically selected to be the most favorable to the attacker and the most difficult for the defense.

The \textbf{SA (simulated annealing)} variant also follows Algorithm~\ref{alg:rand}, but uses a dynamic acceptance probability based on a simulated annealing schedule. The initial temperature is set to 25 and decays by a factor of 0.997 after each iteration. If no new candidate is accepted for 20 consecutive steps, the temperature is reset to 25. The acceptance probability is given by:
\[
Prob = \exp\left( -\frac{ \mathcal{L}_{\text{new}} - \mathcal{L} }{ \texttt{tmp} } \right)
\]
where $\mathcal{L}_{\text{new}}$ is the loss value of the worse candidate, $\mathcal{L}$ is the loss of the current best candidate, and \texttt{tmp} is the current temperature. These settings were also empirically chosen to favor the attacker and stress the defense.

All experiments were run on NVIDIA RTX 3090 GPUs, and other CUDA devices are expected to work as well. For methods involving randomness, we use the default seed 19260817 (a prime number) and run each experiment once.

\section{ASR Trend during Attacks}
\label{sec:asr}
As a supplement to Table~\ref{table:big}, Figure~\ref{fig:asr} shows how the attack success rate (ASR)—that is, the proportion of initially clean samples that have been successfully turned into AEs—evolves over attack iterations for different defenses. 
For the WideResNet-50 model, we plot the worst-case scenario for each defense: reverse tactic for AAA, standard tactic for RND ($\nu = 0.01$) and for the undefended model, and SA tactic for DLD. 
Note that the vertical axis represents the attacker’s success rate, where lower values indicate stronger defense performance.

\begin{figure}
  \centering
  \includegraphics[width=0.5 \textwidth]{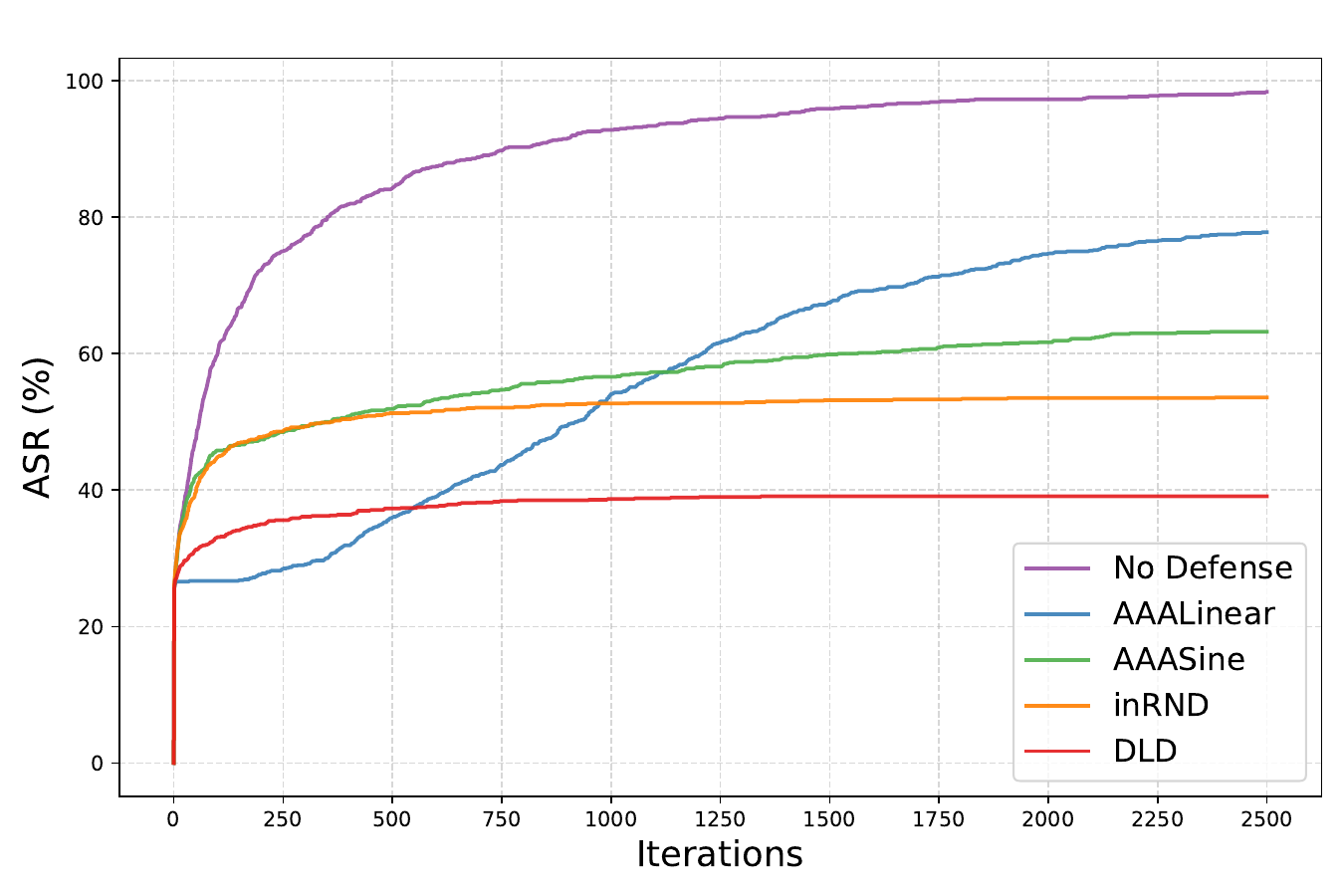}
  \caption{Attack success rate (ASR) over attack iterations.}
  \label{fig:asr}
\end{figure}

\begin{table*}[t]
  \renewcommand{\arraystretch}{0.9}
  \caption{The test accuracy under attacks and defenses (\%)}
  \label{table:l2}
  \centering
  \begin{tabular}{c c c c c c c c c c}
\toprule
\multirow{3}{*}{Model} & \multicolumn{2}{c}{Attack} & \multicolumn{7}{c}{Defense} \\ \cmidrule(lr){2-3} \cmidrule(lr){4-10}
 & \multirow{2}{*}{generator $G$} & \multirow{2}{*}{tactic} & \multirow{2}{*}{None} & \multicolumn{2}{c}{RND} & \multicolumn{2}{c}{AAA} & \multicolumn{2}{c}{DLD} \\ 
 &  &  &  & 0.01 & 0.02 & sine & linear & rand & determine \\ \midrule

\multirow{5}{*}{WideResNet} 
 & \multicolumn{2}{c}{~ ~ ~ ~None} & 100 & 97.4 & 95.8 & 100 & 100 & 100 & 100 \\ 
 & \multirow{4}{*}{Square Attack} & standard & 8.1 & \underline{46.8} & \underline{57.1} & 70.7 & 87.4 & 84.6 & 85.9 \\ 
 &  & explore & 50.6 & 59.5 & 62.1 & 69.3 & 81.8 & 72.2 & 73.2 \\ 
 &  & SA & 69.8 & 68 & 66.6 & 67.8 & 71.5 & \underline{70.0} & \underline{70.6} \\ 
 &  & reverse & 16.2 & 60.5 & 64.8 & \underline{50.6} & \underline{26.9} & 85.6 & 85.9 \\ \midrule

\multirow{5}{*}{RegNet-Y} 
 & \multicolumn{2}{c}{~ ~ ~ ~None} & 100 & 97.7 & 95.7 & 100 & 100 & 100 & 100 \\ 
 & \multirow{4}{*}{Square Attack} & standard & 8.7 & \underline{54.6} & \underline{56.0} & 70.7 & 90.1 & 85.9 & 87.6 \\ 
 &  & explore & 54.9 & 64.4 & 61.8 & 71.6 & 84.2 & 79.1 & 77.8 \\ 
 &  & SA & 74.0 & 69.2 & 65.9 & 74.4 & 76.2 & \underline{76.1} & \underline{75.1} \\ 
 &  & reverse & 15.5 & 64.2 & 64.5 & \underline{51.9} & \underline{27.3} & 87.0 & 87.7 \\ \midrule

 \multirow{5}{*}{MaxViT-T} 
 & \multicolumn{2}{c}{~ ~ ~ ~None} & 100 & 98.6 & 97.5 & 100 & 100 & 100 & 100 \\ 
 & \multirow{4}{*}{Square Attack} & standard & 18.5 & \underline{60.9} & \underline{65.5} & 80.6 & 92.2 & 89.2 & 90.5 \\ 
 &  & explore & 70.0 & 73.7 & 71.6 & 79.9 & 88.5 & 83.6 & 83.3 \\ 
 &  & SA & 81.6 & 79.3 & 74.7 & 80.2 & 83.7 & \underline{81.7} & \underline{81.5} \\ 
 &  & reverse & 37.3 & 71.6 & 72.3 & \underline{65.3} & \underline{43.9} & 90.7 & 90.4 \\ \bottomrule
 
  \end{tabular}
\end{table*}

\section{Numerical Results with $\ell_2$ Noise Budget}
\label{sec:l2}
In the main text, we reported experimental results using an $\ell_\infty$ noise budget. Here, we present results evaluated under an $\ell_2$ noise budget.
We use 1,000 correctly classified ImageNet samples, randomly selecting one image from each class. The experiments are conducted with the PyTorch-pretrained WideResNet-50, RegNet-Y 1.6GF, and MaxViT-T models. AEs are generated using Square Attack with an $\ell_2$ noise budget of $\epsilon_\text{n} = 10$ and a query budget of $n = 1500$.
The under-attack accuracy of different defenses is summarized in Table~\ref{table:l2}. For each defense, the lowest accuracy across all attack tactics is underlined to indicate the worst-case performance under adaptive attacks.
Similar to the observations in Section~\ref{sec:results}, DLD achieves the best worst-case performance among all evaluated defenses.

\newpage

\end{document}